%% file: Grassmann.tex
\documentclass[11pt]{amsart}
\pdfoutput=1
\usepackage[utf8]{inputenc}
\usepackage[T1]{fontenc}
\usepackage{amsmath,amssymb,amsthm,mathtools}
\usepackage{graphicx}
\usepackage[margin=1in]{geometry}
\usepackage{hyperref}
\usepackage{enumitem}
\usepackage{booktabs}
\usepackage{array}
\usepackage{listings}
\usepackage{xcolor}
\usepackage{tikz}
\usepackage[final,protrusion=true,expansion=false]{microtype}

\hypersetup{
    colorlinks=true,
    linkcolor=blue!50!black,
    citecolor=blue!50!black,
    urlcolor=blue!50!black
}

\theoremstyle{plain}
\newtheorem{theorem}{Theorem}[section]
\newtheorem{proposition}[theorem]{Proposition}
\newtheorem{lemma}[theorem]{Lemma}
\newtheorem{corollary}[theorem]{Corollary}

\theoremstyle{definition}
\newtheorem{definition}[theorem]{Definition}

\theoremstyle{remark}
\newtheorem{remark}[theorem]{Remark}

\newcommand{\R}{\mathbb{R}}
\newcommand{\RP}{\mathbb{RP}}
\newcommand{\Sphere}{\mathbb{S}}
\newcommand{\Gr}{\mathrm{Gr}}

\newcommand{\OO}{\mathrm{O}}
\newcommand{\Sym}{\mathrm{Sym}}
\newcommand{\im}{\mathrm{im}}
\newcommand{\rank}{\mathrm{rank}}

\newcommand{\spn}{\mathrm{span}}
\newcommand{\Symp}{\Sym^{+}}
\newcommand{\Sigmapt}{\Sigma^{\mathrm{pure}}_{tt}}
\newcommand{\Sigmabt}{\Sigma^{\mathrm{blur}}_{tt}}
\newcommand{\Sigmaept}{\widetilde{\Sigma}^{\mathrm{pure}}_{tt}}
\newcommand{\Sigmaebt}{\widetilde{\Sigma}^{\mathrm{blur}}_{tt}}
\newcommand{\Sigmafour}{\Sigma_{4\mathrm{D}}}
\newcommand{\Sigmathree}{\Sigma_{3\mathrm{D}}}
\newcommand{\Sigmacond}{\Sigma_{3\mathrm{D}|t}}
\newcommand{\Sigmaecond}{\widetilde{\Sigma}_{3\mathrm{D}|t}}
\newcommand{\Sigmarender}{\Sigma^{\mathrm{render}}_{3\mathrm{D}}}
\newcommand{\Vthree}{V_{3\mathrm{D}}}
\newcommand{\cworld}{c_{\mathrm{world}}}

\newcommand{\paperfig}[2]{%
  \IfFileExists{#1}{\includegraphics[width=#2]{#1}}%
  {\fbox{\parbox[c][3.5cm][c]{#2}{\centering\ttfamily\detokenize{#1}\\[2pt](image not found; place file here)}}}}

\definecolor{codegray}{rgb}{0.5,0.5,0.5}
\definecolor{codegreen}{rgb}{0,0.5,0}
\definecolor{codepurple}{rgb}{0.58,0,0.82}
\definecolor{backcolour}{rgb}{0.97,0.97,0.97}

\lstdefinestyle{pythonstyle}{
    backgroundcolor=\color{backcolour},
    commentstyle=\color{codegreen},
    keywordstyle=\color{magenta},
    numberstyle=\tiny\color{codegray},
    stringstyle=\color{codepurple},
    basicstyle=\ttfamily\footnotesize,
    breakatwhitespace=false,
    breaklines=true,
    captionpos=b,
    keepspaces=true,
    numbers=left,
    numbersep=5pt,
    showspaces=false,
    showstringspaces=false,
    showtabs=false,
    tabsize=2,
    language=Python
}

\title[Grassmannian Splatting I]{Grassmannian Splatting I:\\
Moving rank-2 Spacetime Surfels for Dynamic Scene Rendering}

\author{Aaron~Maurice~Berman}
\author{Shantanu~Dave}

\date{July 2026}

\begin{document}

\begin{abstract}
We introduce Grassmannian splatting, a dynamic scene representation whose
primitives are Gaussians supported on 3-planes in spacetime
$\R^4$: generically, spatial 2-planes in uniform translation along their
normals. Each primitive carries a unit normal
$n \in \Sphere^3/\{\pm 1\} \cong \Gr(3,4)$ and an unconstrained factor
$L \in \R^{4 \times 3}$, with covariance
\[
  \Sigmafour = (P_n L)(P_n L)^T, \qquad P_n = I - n n^T.
\]
For generic $L$ and $n \neq \pm e_0$, conditioning on time returns a rank-2
surfel at every frame. The normal of the disk and its velocity along that
normal are read off from $n$; the disk shape and the tangential drift of its
center are set by $L$. Existing native 4D Gaussian
splatting methods \cite{yang2023gs4d,duan20244drotorgaussiansplattingefficient} slice full-rank spacetime
covariances, so their per-frame primitive is a volumetric ellipsoid; since
conditioning lowers rank by exactly one, a rank-2 surfel in the slice requires
a rank-3 spacetime covariance, and the parameterization above realizes exactly
these. The motion model is closed form, i.e. no deformation field is learned, and
no custom CUDA is required: the conditioned disk
feeds a standard 3DGS rasterizer through its precomputed-covariance interface.
A soft clamp in the Schur denominator regularizes the static orientation and
continuously bridges rank-3 static and rank-2 dynamic behavior, so static and
moving primitives form a single continuous family. On the 17 HyperNeRF scenes of MonoDyGauBench, training is fastest among
all compared methods (4.9 to 5.6 times faster than the strongest quality
baselines), while ranking second in
PSNR, MS-SSIM, and LPIPS.
Code: \url{https://github.com/PaulCelanCoding/grassmannian-splatting}

\end{abstract}

\maketitle

% =====================================================================
% Introduction
% =====================================================================

\section{Introduction}\label{sec:intro}
 
The Grassmannian $\Gr(3,4)$ is a moduli space of constant-velocity plane
motions: every 3-plane in spacetime $\R^4$, with the single exception of the
static slice $[e_0]$, is a spatial 2-plane in uniform translation along its
normal (Lemma~\ref{prop:slicing}). We take these motions as the primitives of
dynamic scene rendering: each splat is a Gaussian supported on such a plane.
Slicing it at any time yields a rank-2 surfel whose normal and normal velocity
are read off from the underlying point of $\Gr(3,4)$. Concretely, we
parameterize the plane by its unit normal $n \in \Sphere^3/\{\pm 1\}$ and an
unconstrained factor $L \in \R^{4 \times 3}$, and set
\begin{equation}\label{eq:headline}
  \Sigmafour = (P_n L)(P_n L)^T, \qquad P_n = I - n n^T.
\end{equation}
Since $\Sigmafour$ depends on $n$ only through the projector $P_n = P_{-n}$, it
is a well-defined function on $\Gr(3,4) \cong \Sphere^3/\{\pm 1\}$; the
projector is the standard embedding of the Grassmannian into the symmetric
matrices. The covariance is PSD of rank at most 3 with kernel direction $n$
(Lemma~\ref{prop:rank-kernel}), and for generic $L$ and $n \neq \pm e_0$,
conditioning on time yields a spatial covariance of rank exactly 2: a disk in
$\R^3$, at every frame, with no per-frame flatness constraint and no learned
deformation field.

 \begin{figure}[t]
\centering
\paperfig{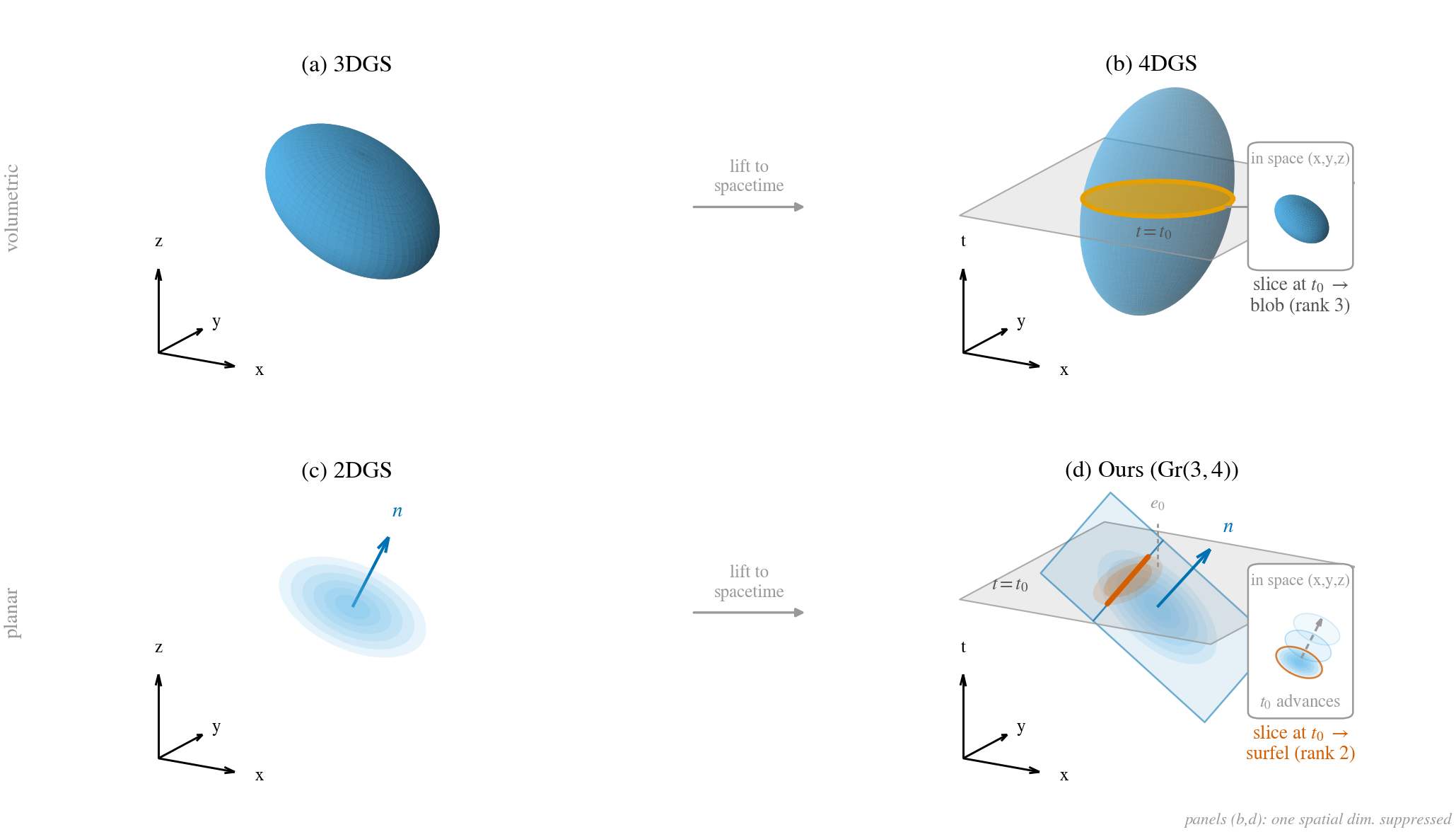}{\textwidth}
\caption{Primitive types. Top (volumetric): (a) the 3D ellipsoid of 3DGS \cite{kerbl3Dgaussians}; (b) its native 4D analogue \cite{yang2023gs4d,duan20244drotorgaussiansplattingefficient}, whose spacetime slice at $t = t_0$ is again a rank-3 blob. Bottom (planar): (c) the rank-2 disk of 2DGS \cite{Huang2DGS2024}, with normal $n$; (d) ours, a Gaussian supported on a 3-plane in $\R^4$ with normal $n \in \Sphere^3/\{\pm 1\} \cong \Gr(3,4)$, whose slice at $t = t_0$ is a rank-2 surfel that translates along its spatial normal as $t_0$ advances. Panels (b,d) suppress one spatial dimension.}
\label{fig:zoo}
\end{figure}

The primitive fills a specific cell of the design space. Static splatting
offers a volumetric primitive, the rank-3 ellipsoid of 3DGS \cite{kerbl3Dgaussians},
and a planar one, the rank-2 surfel of 2DGS \cite{Huang2DGS2024}. Dynamic
splatting adds the temporal axis in three ways. Deformation- or tracking-based methods carry a
primitive in a canonical frame and warp it into each frame
\cite{Wu_2024_CVPR,Yang_2024_CVPR}; among these, space-time 2D Gaussian
splatting (ST-2DGS) \cite{wang2024spacetime2dgaussiansplatting} warps rank-2
surfels. Per-primitive trajectory methods attach an explicit parametric motion
to each 3D Gaussian \cite{Li_2024_CVPR}. Native methods place the primitive
in spacetime $\R^4$ and recover each frame by conditioning on time
\cite{yang2023gs4d,duan20244drotorgaussiansplattingefficient}; their spacetime (or higher-dimensional
\cite{gao20257dgsunifiedspatialtemporalangulargaussian,wu2025orientationanchoredhypergaussian4dreconstruction}) covariance is full rank, so the conditioned
primitive is a rank-3 ellipsoid. Our primitive is native and rank-deficient:
to our knowledge the first spacetime Gaussian whose slices are surfels. Its
motion is induced by the slicing (the disk of a 3-plane-supported Gaussian
translates at constant velocity), so the per-primitive temporal coherence that
trajectory methods build in through an explicit curve is here a property of
the plane itself.
 
The slicing step itself is shared with the full-rank native methods
\cite{yang2023gs4d,duan20244drotorgaussiansplattingefficient}: the same conditional-Gaussian computation in every
case. What it returns is fixed by the Schur rank identity
(Lemma~\ref{lem:schur-rank}): conditioning on one coordinate lowers the rank by
exactly one, so a rank-2 surfel in the slice is equivalent to a rank-3
spacetime covariance with $\Sigmapt > 0$. Equation~\eqref{eq:headline}
parameterizes exactly these covariances (the PSD matrices with image in the
3-plane $E_n$) globally over $\Gr(3,4)$
(Proposition~\ref{prop:effective-dofs}).
 
The parameterization is global: the plane enters only through the projector
$P_n$, never through a local chart of $\Gr(3,4)$, so the rank structure of
\eqref{eq:headline} holds without case distinctions
(Section~\ref{sec:geometry}). The price is redundancy in $L$ (an $\OO(3)$
right action together with the projected-out $n$-direction), leaving
$\Sigmafour$ with 6 effective degrees of freedom against the 12 parameters of
$L$ (Proposition~\ref{prop:effective-dofs}). The redundancy is invisible to the
photometric objective at fixed $n$ and requires no projection step
(Section~\ref{subsec:effective-dofs}).
 
Conditioning on $x_0 = t_0$ yields the surfel, and it can be read in three
equivalent ways. Geometrically, the 3-plane slices to a 2-plane whose normal is
the spatial part $n_{1:}$ of the spacetime normal (Lemma~\ref{prop:slicing}).
Probabilistically, the kernel of the Schur complement is spanned by the same
$n_{1:}$, and the conditional mean lands on the geometric slice
(Lemma~\ref{lem:schur-on-slice}): the two derivations produce the same disk,
the same normal, and the same center. In latent terms, writing
$\Sigmafour = M M^T$ with $M = P_n L$, the first row $m_0$ of $M$ is the latent
direction that controls time, and conditioning is exactly the projection that
removes it; the lower rows $M_s$ map the remaining 2-dimensional latent
subspace to the disk, and the motion of the center is governed by $M_s m_0$
(Section~\ref{sec:world-space-conditioning}). This separates the two roles:
$n$ selects the plane and its normal velocity, while $L$ sets the disk shape
and the tangential drift of its center
(Remark~\ref{rmk:normal-vs-tangential}).
 
The time axis stratifies $\Gr(3,4)$ into an open dense set of moving disks, an
$\RP^2$ of planes containing the time axis (their spatial support is
time-independent), and the single point $n = \pm e_0$. At $n = \pm e_0$ the
support plane is the spatial slice itself, the temporal variance $\Sigmapt$
vanishes, and the conditioning map is singular: its limit as $n \to e_0$
depends on the direction of approach (Proposition~\ref{prop:bridge}). A soft
clamp in the Schur denominator regularizes this singularity and, below a tilt
threshold, interpolates continuously to static rank-3 behavior
(Corollary~\ref{cor:bridge}). Static and dynamic primitives thus form a single
continuous family (Section~\ref{sec:numerical-stability}).
 
Nothing in the pipeline requires a custom kernel. The conditioned rank-2 disk,
lifted by a small isotropic $\sigma^2_{\mathrm{lift}} I_3$ for invertibility,
is passed to an unmodified 3DGS rasterizer through the \texttt{cov3D\_precomp}
interface, and a temporal opacity weight derived from $\Sigmapt$ fades each
primitive in and out around its lifetime, so appearing and disappearing content
is representable in principle without creating or destroying primitives
mid-sequence (Section~\ref{sec:rendering}). All of the construction is confined to
the covariance preprocessing; the rendering kernel is untouched.

\subsection{Notation}
We work in $\R^4$ with coordinates $x = (x_0, X)$, where $x_0 \in \R$ is time and $X = (x_1, x_2, x_3) \in \R^3$ is space. The standard basis is $\{e_0, e_1, e_2, e_3\}$ with $e_0 = (1,0,0,0)^T$. Inner products are Euclidean; $I_k$ is the $k \times k$ identity; $\Symp_k$ is the $k \times k$ PSD cone. For a matrix $M$, $\im(M)$ is its column space and $\ker(M)$ its kernel. The Schur complement of $\begin{pmatrix} A & B \\ B^T & D \end{pmatrix}$ with $A$ invertible is $D - B^T A^{-1} B$.

% =====================================================================
\section{The Geometric Setting}\label{sec:geometry}
% =====================================================================

In this section, we recall the relevant objects. The ambient space is $\R^4 = \R \times \R^3$, treated as a Euclidean inner product space with a distinguished basis vector $e_0$ (the time axis).

\subsection{The Grassmannian $\Gr(3,4)$ and projectors}

\begin{definition}[Grassmannian]
Let $\Gr(3,4)$ denote the manifold of 3-dimensional linear subspaces of $\R^4$. As a homogeneous space,
\[
  \Gr(3,4) \cong \OO(4)/(\OO(3) \times \OO(1)) \cong \Sphere^3/\{\pm 1\} \cong \RP^3,
\]
a smooth manifold of dimension 3.
\end{definition}

For a unit vector $n \in \Sphere^3 \subset \R^4$, the canonical 3-plane is
\[
  E_n = \{x \in \R^4 : \langle n, x \rangle = 0\}.
\]
Since $E_n = E_{-n}$, the assignment $n \mapsto E_n$ descends to a bijection $\Sphere^3/\{\pm 1\} \to \Gr(3,4)$. The orthogonal projector onto $E_n$ is
\[
  P_n = I - n n^T \in \R^{4 \times 4},
\]
with $P_n^2 = P_n = P_n^T$, $P_n n = 0$, and $P_n x = x$ for all $x \in E_n$. 

\subsection{Slicing}
A 3-plane in spacetime is a moving 2-plane in space.

\begin{lemma}[Slicing law]\label{prop:slicing}
Write $n = (n_0, n_{1:}) \in \R \times \R^3$. If $n_{1:} \neq 0$, the spatial slice of $E_n$ at time $t_0$ is the affine plane
\[
  \{X \in \R^3 : \langle n_{1:}, X \rangle = -t_0 n_0\} \subset \R^3,
\]
with unit normal $n_{1:}/\|n_{1:}\|$ and signed distance $-t_0 n_0 / \|n_{1:}\|$ from the origin. As $t_0$ varies, the plane translates rigidly along its normal with velocity $-n_0 \, n_{1:} / \|n_{1:}\|^2 \in \R^3$ (signed speed $-n_0/\|n_{1:}\|$). In particular, the point of the slice nearest the origin is $-t_0\,\frac{n_0}{\|n_{1:}\|^2}\, n_{1:}$.\end{lemma}

\begin{proof}
The defining equation $\langle n, x \rangle = 0$ at $x_0 = t_0$ reads $n_0 t_0 + \langle n_{1:}, X \rangle = 0$, giving the stated affine equation. The normal direction $n_{1:}/\|n_{1:}\|$ is independent of $t_0$, so the plane translates rigidly.
\end{proof}

\begin{remark}[Affine version]\label{rmk:affine}
The Gaussian primitive will be supported on a translate $\mu + E_n$ for a per-Gaussian anchor $\mu = (v_0, V)$. The slicing law transfers verbatim with $t_0 \mapsto t_0 - v_0$, $X \mapsto X - V$: the spatial slice of $\mu + E_n$ at time $t_0$ is
\[
  \{X \in \R^3 : \langle n_{1:}, X - V \rangle = -n_0 (t_0 - v_0)\},
\]
centered at $V$ when $t_0 = v_0$. Section~\ref{sec:world-space-conditioning} derives a probabilistic mean for the slice via conditioning, and Lemma~\ref{lem:schur-on-slice} verifies that the two descriptions agree.
\end{remark}

\begin{remark}[Degenerate cases]\label{rmk:degenerate}
If $n_0 = 0$, the plane $E_n$ contains the time axis: every spatial slice is the same 2-plane in $\R^3$, so the spatial support is time-independent. We will see in  Section \ref{sec:world-space-conditioning} that the Gaussian itself need not be frozen, since the conditional mean can drift tangentially within the fixed plane; this is encoded by the covariance cross-block $\cworld$ from \eqref{eq:block} rather than by $n$. Conversely, if $n_{1:} = 0$ (i.e.\ $n = \pm e_0$), the plane $E_n$ is exactly the spatial slice $\{x_0 = 0\}$. The pure temporal variance $\Sigmapt$ vanishes, so the exact conditioning formulas are singular. This is the static orientation, handled by the numerical bridge of Section~\ref{sec:numerical-stability}.
\end{remark}

% =====================================================================
\section{The Gaussian Primitive}\label{sec:primitive}
% =====================================================================

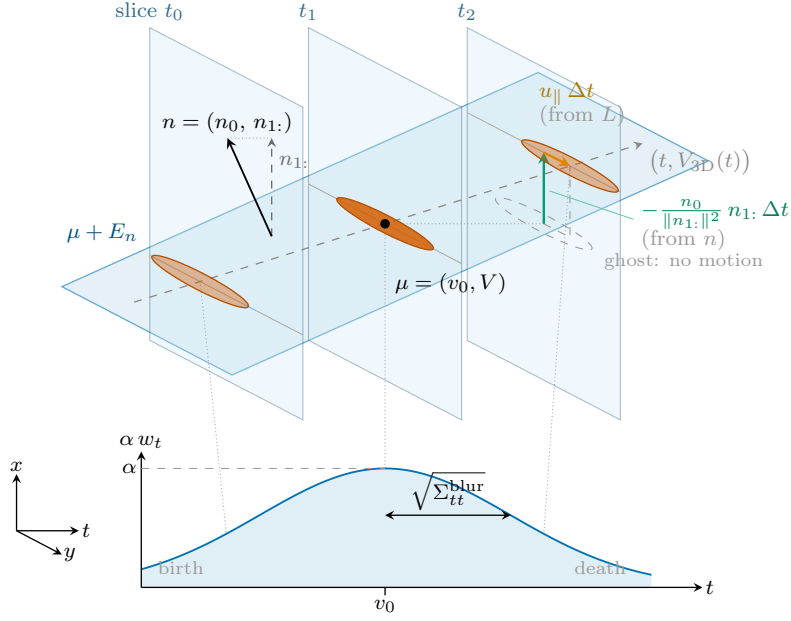
\begin{figure}[t]
\centering
\input{fig_anatomy.tex}
\caption{Anatomy of a single primitive. One spatial dimension is suppressed, so spacetime is drawn as $(t, x, y)$; the translucent planes are the frame slices $\{t = t_i\}$. The Gaussian lives on the affine plane $\mu + E_n$, and its frame slice is a segment on the intersection line $(\mu + E_n) \cap \{t = t_i\}$, in the full model a rank-2 disk (Lemma~\ref{lem:block-consequences}), with fill encoding the temporal weight $w_t$. The spacetime normal $n$ points out of the support plane; its spatial part $n_{1:}$ is the in-slice normal of the sliced disk. Carrying the $t_1$-disk to $t_2$ without motion (dashed ghost) and comparing with the true $t_2$-disk decomposes the displacement over $\Delta t = t_2 - t_1$ inside the $t_2$-slice: normal motion along $n_{1:}$, determined by $n$ alone (green), plus tangential drift $u_\parallel$ along the slice, encoded in $L$ through the cross-block $\cworld$ (orange): both arrows are velocities times $\Delta t$; Remark~\ref{rmk:normal-vs-tangential}. The dashed grey line is the center worldline $(t, \Vthree(t))$, which lies in the support plane (Lemma~\ref{lem:schur-on-slice}). Bottom: the temporal envelope $\alpha^{\mathrm{eff}}(t) = \alpha\, w_t$ expresses birth and death of scene content; the double arrow marks the $1\sigma$ half-width $\sqrt{\Sigmabt}$, where $w_t = e^{-1/2} \approx 0.61$ of its peak (Section~\ref{subsec:effective-opacity}). Schematic with illustrative parameter values, not trained data.}
\label{fig:anatomy}
\end{figure}

We do not assign a covariance to a primitive directly; the covariance is computed from the parameters below in every forward pass.

\subsection{Parameters}
Each Gaussian carries:
\begin{itemize}
  \item $n \in \Sphere^3/\{\pm 1\}$, the plane normal, 3 DOF, stored as unconstrained $n_{\mathrm{raw}} \in \R^4$ and renormalized in the forward pass;
  \item $L \in \R^{4 \times 3}$, an unconstrained factor;
  \item $\mu = (v_0, V) \in \R \times \R^3 = \R^4$, the mean in spacetime;
  \item $\alpha \in [0, 1]$, the base opacity, stored as a logit;
  \item a color, either a constant RGB triple or spherical harmonic coefficients;
  \item a per-Gaussian rank-lift variance $\sigma^2_{\mathrm{lift},k}$, stored as a logit; the blur scalars $\sigma^2_{\mathrm{pix}}, \sigma^2_{\mathrm{tmp}}$ are global and shared by all primitives (Section~\ref{subsec:blur-trio}).
\end{itemize}

\subsection{The 4D covariance}

\begin{definition}[Projected factor and 4D covariance]
The projected factor is
\begin{equation}\label{eq:Ln}
  L_n = P_n L = L - n(n^T L) \in \R^{4 \times 3},
\end{equation}
whose three columns lie in $E_n$. The 4D covariance is
\begin{equation}\label{eq:Sigma4D}
  \Sigmafour = L_n L_n^T = (P_n L)(P_n L)^T \in \Symp_4.
\end{equation}
\end{definition}

\begin{lemma}[Rank and kernel of $\Sigmafour$]\label{prop:rank-kernel}
$\Sigmafour$ is PSD with $\rank(\Sigmafour) \leq 3$ and $\spn(n) \subseteq \ker(\Sigmafour)$. Consequently $\im(\Sigmafour) \subseteq E_n$.
\end{lemma}

\begin{proof}
Since $\Sigmafour = L_n L_n^T$, it is positive semidefinite. For the kernel: $L_n^T n = (P_n L)^T n = L^T P_n n = 0$, so $\Sigmafour n = 0$ by symmetry. The image is $\ker(\Sigmafour)^\perp \subseteq \spn(n)^\perp = E_n$.
\end{proof}

This is the central architectural difference from \cite{yang2023gs4d,duan20244drotorgaussiansplattingefficient}: their $\Sigmafour$ is generically rank 4, while ours is at most rank 3 by construction.

\subsection{Effective degrees of freedom of $L$}\label{subsec:effective-dofs}
$L$ has 12 entries, but $\Sigmafour$ depends on $L$ only through its image in $E_n$ and only modulo a right action of $\OO(3)$.

\begin{proposition}[Effective DOFs of $\Sigmafour$]\label{prop:effective-dofs}
For fixed $n$, the set of covariances realizable as $\Sigmafour = (P_n L)(P_n L)^T$, $L \in \R^{4 \times 3}$, is the set of PSD matrices with image contained in $E_n$, a cone of dimension 6. In particular $\Sigmafour$ has at most 6 effective degrees of freedom, against the 12 entries of $L$.
\end{proposition}

\begin{proof}
The set of PSD rank-$\leq 3$ matrices with image in the 3-dimensional subspace $E_n$ is a copy of $\Symp_3$, of dimension 6. Conversely, any such matrix factorizes as $L_n L_n^T$ with the columns of $L_n$ in $E_n$, and $L = L_n$ satisfies $P_n L = L_n$. So the map $L \mapsto \Sigmafour$ surjects onto a 6-dimensional set.
\end{proof}

\subsubsection{Effect of Redundancies on optimization} Let $\mathcal L$ denote the loss function described in \ref{subsec:loss}. We first investigate $\partial \mathcal  L/\ \partial  L$.
The 6 redundant directions decompose as follows. Write $L = L_n + n a^T$ with $a = L^T n \in \R^3$. The component $na^T$ (3 entries) is annihilated by $P_n$, and it does not contribute to $\Sigmafour$, and hence the photometric loss is independent of $a$ when only $L$ varies; its gradient $\frac{\partial\mathcal L}{\partial L}$ is lifted from $\frac{\partial\mathcal L}{\partial L_n}$. 

The remaining redundancy is the right action $L \mapsto L Q$, $Q \in \OO(3)$, which leaves $L_n L_n^T$ invariant. The loss $\mathcal L$ remains constant along the orbit of this action, so the optimizer settles into an arbitrary orbit representative. No constraint or projection step is used. 

\subsection{The blur trio: $\sigma^2_{\mathrm{pix}}, \sigma^2_{\mathrm{tmp}}, \sigma^2_{\mathrm{lift}}$}\label{subsec:blur-trio}
Standard EWA splatting \cite{zwicker2001ewa} carries a single isotropic blur parameter. Three distinct quantities arise in our pipeline; all are variances and all are non-negative, but they enter at different stages, and conflating them leads to incorrect formulas.

\paragraph{(1) Pixel-domain blur $\sigma^2_{\mathrm{pix}}$.} The standard EWA blur, added as $\sigma^2_{\mathrm{pix}} I_2$ to the projected 2D covariance after perspective transformation. Handled entirely by the rasterizer. Default $\sigma^2_{\mathrm{pix}} = 1$ (pixel units).

\paragraph{(2) Temporal modeling floor $\sigma^2_{\mathrm{tmp}}$.} Added to $\Sigmapt$ in the temporal weight $w_t$ (Section~\ref{subsec:effective-opacity}). It enforces a minimum visible duration for every Gaussian. Default $\sigma^2_{\mathrm{tmp}} = 0$, so that the data dictates the temporal width. Numerical stability when $\Sigmapt \to 0$ and $\sigma^2_{\mathrm{tmp}} = 0$ simultaneously is handled by a separate numerical floor $\varepsilon$ (Section~\ref{sec:numerical-stability}). Units: frames squared.

\paragraph{(3) 3D rank lift $\sigma^2_{\mathrm{lift}}$.} In the exact model, the conditioned spatial covariance has rank exactly 2 and is singular as a $3 \times 3$ matrix. CUDA EWA kernels require an invertible 3D covariance, so we add $\sigma^2_{\mathrm{lift},k} I_3$ before rasterization. The lift is a \emph{per-Gaussian learnable} parameter (a logit mapped through a softplus, Section~\ref{sec:gradients}), initialized at $\sigma^2_{\mathrm{lift}} = 10^{-4}$ in scene units (roughly 1 cm in a typical scene, below pixel resolution at standard rendering distances), from which the optimizer can thicken or thin individual primitives. This parameter is specific to rank-aware splatting: in the rank-4 methods \cite{yang2023gs4d,duan20244drotorgaussiansplattingefficient}, the conditioned covariance is already full rank and no lift is needed.

% =====================================================================
\section{World-Space Time Conditioning}\label{sec:world-space-conditioning}
% =====================================================================

We now condition $\mathcal{N}(\mu, \Sigmafour)$ on $x_0 = t_0$ to obtain a 3D Gaussian on $\R^3$ in world coordinates, and prove that the resulting covariance has rank 2. Decompose $\Sigmafour$ along the time direction $e_0$:
\begin{equation}\label{eq:block}
  \Sigmafour = \begin{pmatrix} \Sigmapt & \cworld^T \\ \cworld & \Sigmathree \end{pmatrix},
  \qquad \mu = \begin{pmatrix} v_0 \\ V \end{pmatrix},
\end{equation}
where $\Sigmapt \in \R$, $\cworld \in \R^3$, $\Sigmathree \in \R^{3 \times 3}$. The superscript ``pure'' distinguishes this from the blurred quantity $\Sigmabt = \Sigmapt + \sigma^2_{\mathrm{tmp}}$ used in the temporal weight.

\begin{remark}[When does $\Sigmapt$ vanish?]\label{rmk:tt-vanish}
Since $\Sigmapt = e_0^T \Sigmafour e_0$ and $\Sigmafour$ is PSD, $\Sigmapt = 0$ iff $\Sigmafour e_0 = 0$. Under the standing assumption $\rank(\Sigmafour) = 3$, this means $e_0 \in \spn(n)$, i.e.\ $n = \pm e_0$. So $\Sigmapt > 0$ exactly when the plane is not the static slice; the singularity of the conditioning formulas occurs only at the static orientation.
\end{remark}

The kernel constraint $\Sigmafour n = 0$ defines an algebraic cone; we now extract its block-level consequences.

\begin{lemma}[Block consequences of the plane constraint]\label{lem:block-consequences}
With $a = \Sigmapt$, $c = \cworld$, $S = \Sigmathree$, $n = (n_0, n_{1:})$, the relation $\Sigmafour n = 0$ is equivalent to
\begin{equation}\label{eq:kernel-pair}
  a n_0 + \langle n_{1:}, c \rangle = 0, \qquad c n_0 + S n_{1:} = 0.
\end{equation}
If $a > 0$, the Schur complement $\Sigmacond = S - c c^T / a$ satisfies
\begin{equation}\label{eq:schur-kernel}
  \Sigmacond n_{1:} = 0.
\end{equation}
If $\rank(\Sigmafour) = 3$ and $n_{1:} \neq 0$, then $\rank(\Sigmacond) = 2$ and $\ker(\Sigmacond) = \spn(n_{1:})$.
\end{lemma}

\begin{proof}
Writing $\Sigmafour n$ in block form gives \eqref{eq:kernel-pair}. For \eqref{eq:schur-kernel}, $S n_{1:} = -c n_0$ and $\langle n_{1:}, c\rangle = -a n_0$ together give
$\Sigmacond n_{1:} = -c n_0 - c(-a n_0)/a = 0$.
Under $\rank(\Sigmafour) = 3$, Lemma~\ref{lem:schur-rank} gives $\rank(\Sigmacond) = 2$, so $\dim\ker(\Sigmacond) = 1$, and $n_{1:} \neq 0$ generates it.
\end{proof}

\subsection{Conditioning on time}
Provided $\Sigmapt > 0$, the conditional Gaussian on $\{x_0 = t_0\}$ is $\mathcal{N}(\Vthree(t_0), \Sigmacond)$ with
\begin{align}
  \Vthree(t_0) &= V + \cworld \cdot (t_0 - v_0) / \Sigmapt, \label{eq:Vthree}\\
  \Sigmacond  &= \Sigmathree - \cworld \cworld^T / \Sigmapt. \label{eq:Sigmacond}
\end{align}
This is the standard conditional formula for jointly Gaussian variables, the same one used in \cite{yang2023gs4d,duan20244drotorgaussiansplattingefficient}; the difference here is that it is applied to a deliberately rank-deficient covariance. Note that $\Sigmacond$ does not depend on $t_0$; only the mean does.

\begin{remark}[Normal motion vs.\ tangential drift]\label{rmk:normal-vs-tangential}
The plane normal $n$ determines the motion of the support slice only in the normal direction. The conditional mean has velocity $\dot{V}_{3\mathrm{D}} = \cworld/\Sigmapt$, and the kernel identity $\langle n_{1:}, \cworld\rangle = -n_0\Sigmapt$ gives $\langle n_{1:}, \dot V_{3\mathrm{D}}\rangle = -n_0$. So when $n_{1:} \neq 0$,
\[
  \dot V_{3\mathrm{D}} = -\frac{n_0}{\|n_{1:}\|^2} n_{1:} + u_\parallel, \qquad u_\parallel \in n_{1:}^\perp.
\]
The first term is the normal velocity of the sliced plane (Lemma~\ref{prop:slicing}); the tangential component $u_\parallel$ is encoded also by the cross-block $\cworld$, hence by $L$, not just by $n$. A full description of the motion requires both $n$ and $L$.
\end{remark}

\subsection{Rank 2 by construction}\label{subsec:rank-summary}
Throughout we restrict to the generic case $\rank(\Sigmafour) = 3$. By Lemma~\ref{lem:block-consequences}, when $\Sigmapt > 0$ (equivalently $n \neq \pm e_0$, Remark~\ref{rmk:tt-vanish}) the conditioned covariance $\Sigmacond$ has rank exactly 2, with kernel $\spn(n_{1:})$.

\subsubsection{Geometric reading}
The degenerate Gaussian $\mathcal{N}(\mu, \Sigmafour)$ is supported on the affine 3-plane $\mu + E_n$, and slicing by $\{x_0 = t_0\}$ restricts the support to the 2-dimensional intersection of Lemma~\ref{prop:slicing}. A nondegenerate Gaussian on a 3-plane, conditioned on a nontrivial linear coordinate, remains nondegenerate on the resulting 2-plane. The framework therefore arrives at a 2DGS-style surfel via slicing, with no explicit flatness constraint.

\begin{lemma}[The conditional mean lies on the slice]\label{lem:schur-on-slice}
The conditional mean $\Vthree(t_0)$ lies on the spatial slice of $\mu + E_n$ at time $t_0$, and $\Sigmacond$ has the slice's normal $n_{1:}$ as its kernel direction.
\end{lemma}

\begin{proof}
Using \eqref{eq:Vthree} and the kernel identity $\langle n_{1:}, \cworld \rangle = -n_0 \Sigmapt$ from \eqref{eq:kernel-pair},
\[
  \langle n_{1:}, \Vthree(t_0) - V \rangle = \frac{t_0 - v_0}{\Sigmapt}\, \langle n_{1:}, \cworld\rangle = -n_0(t_0 - v_0),
\]
which is the defining equation of the slice in Remark~\ref{rmk:affine}. The kernel statement is \eqref{eq:schur-kernel}.
\end{proof}

\subsubsection{Latent-space reading}
Let $M \in \R^{4 \times 3}$ have rank 3 (in the pipeline, $M = L_n$), let $m_0^T$ be its first row and $M_s$ its lower three rows. Writing $\Sigmafour = M M^T$, we have $\Sigmapt = \|m_0\|^2$ and, for $m_0 \neq 0$,
\[
  \Sigmacond = M_s\, \Pi_{m_0^\perp}\, M_s^T, \qquad \Pi_{m_0^\perp} = I_3 - \frac{m_0 m_0^T}{\|m_0\|^2}.
\]
The projector $\Pi_{m_0^\perp}$ removes exactly the latent direction $m_0$ that controls the time coordinate: up to the mean, the support of the conditioned Gaussian is the image under $M_s$ of the two-dimensional latent subspace $m_0^\perp \subset \R^3$, and the motion of the center is governed by $M_s m_0$.

This formulation underlies the proof of Proposition~\ref{prop:bridge}.

\subsection{Rank under the stabilized denominator}\label{subsec:stab-rank}
Equations \eqref{eq:Vthree}--\eqref{eq:Sigmacond} use the exact denominator $\Sigmapt$. The implementation replaces it by the soft clamp $\Sigmaept = \sqrt{(\Sigmapt)^2 + \varepsilon^2}$ of Section~\ref{sec:numerical-stability}. This replaces exact Gaussian conditioning by a numerical continuation. The stabilized covariance is
\[
  \Sigmaecond = \Sigmathree - \frac{\cworld\cworld^T}{\Sigmaept}
  = \Sigmacond + \Delta_\varepsilon,
  \qquad
  \Delta_\varepsilon = \Bigl(\frac{1}{\Sigmapt} - \frac{1}{\Sigmaept}\Bigr)\cworld\cworld^T.
\]
The correction $\Delta_\varepsilon$ is PSD of rank at most 1 (the prefactor is non-negative since $\Sigmaept \geq \Sigmapt$), and for $n_0 \neq 0$ its direction $\cworld$ does not lie in $\im(\Sigmacond) = n_{1:}^\perp$, because $\langle n_{1:}, \cworld\rangle = -n_0\Sigmapt \neq 0$ by \eqref{eq:kernel-pair}. Hence for finite $\varepsilon$ and $n_0 \neq 0$, $\Sigmaecond$ is generically of full rank 3; it approaches the exact rank-2 covariance whenever $\Sigmapt \gg \varepsilon$. All exact statements above (rank 2, slice consistency) refer to the unclamped model.

% ====================================================================
\section{Rendering}\label{sec:rendering}
% =====================================================================

The conditional Gaussian captures the spatial shape of the splat at $t_0$ but not its temporal presence: a primitive centered at $v_0 = 0$ with $\sqrt{\Sigmapt} = 1$ frame should have faded out by $t_0 = 100$, and conditioning alone does not produce this fade. We therefore modulate opacity by a temporal weight.

\subsection{Effective opacity and the temporal weight}\label{subsec:effective-opacity}

\begin{definition}
Let $\Sigmabt = \Sigmapt + \sigma^2_{\mathrm{tmp}}$. The temporal weight at frame $t_0$ is
\begin{equation}\label{eq:wt}
  w_{t_0} = \exp\!\bigl(-(t_0 - v_0)^2/(2 \Sigmabt)\bigr) \in (0, 1],
\end{equation}
and the effective opacity is $\alpha^{\mathrm{eff}}(t_0) = \alpha \cdot w_{t_0} \in [0, \alpha]$.
\end{definition}

\begin{remark}\label{rmk:unnormalized}
The weight is deliberately unnormalized.
$w_t$ is the Gaussian \emph{kernel} (peak 1), not the Gaussian \emph{density} (which carries a factor $(2\pi\Sigmabt)^{-1/2}$). The product $\alpha \cdot w_t \cdot \mathcal{N}(\cdot\,; \Vthree, \Sigmacond)$ used in the rendering equation is therefore not the joint 4D density $p(y, t_0) = p(y \mid t_0) p(t_0)$. We drop the normalization to keep $\alpha^{\mathrm{eff}} \in [0, 1]$, as required by alpha compositing. This matches standard 3DGS practice; the missing normalization is absorbed by the learned $\alpha$.
\end{remark}

\begin{remark} Using $\Sigmabt$ in $w_t$ keeps the temporal weight independent of the conditioning, which uses $\Sigmapt$.
Adding $\sigma^2_{\mathrm{tmp}}$ to the denominator of \eqref{eq:Sigmacond} would smear the spatial slice and break the rank-2 structure of Lemma~\ref{lem:block-consequences}. The temporal weight does not need the rank structure, only a positive denominator, so the modeling floor enters only there. The numerical floor $\varepsilon$ (Section~\ref{sec:numerical-stability}) is shared by both denominators and is a separate device.
\end{remark}

\subsection{The full rendering equation}
For each Gaussian $k$ and frame $t_0$, preprocessing produces a triple $({\Vthree}_k(t_0), {\Sigmarender}_k, \alpha^{\mathrm{eff}}_k(t_0))$ with
\begin{equation}\label{eq:Sigma-render}
  {\Sigmarender}_k = {\Sigmacond}_k + \sigma^2_{\mathrm{lift},k} I_3.
\end{equation}
This triple, together with a color $c_k$, is fed to a standard 3DGS rasterizer, which applies the view transform $(R(t_0), c(t_0))$, the EWA perspective Jacobian, the pixel blur $\sigma^2_{\mathrm{pix}} I_2$, and front-to-back alpha compositing:
\begin{equation}\label{eq:rendering}
  C(y \mid t_0) = \sum_{k \in \text{sorted}} c_k \alpha^{\mathrm{eff}}_k(t_0) p_k(y \mid t_0) \prod_{j < k}\!\bigl(1 - \alpha^{\mathrm{eff}}_j(t_0) p_j(y \mid t_0)\bigr).
\end{equation}
Here $p_k(y \mid t_0)$ is the projected unnormalized Gaussian kernel (peak 1), consistent with Remark~\ref{rmk:unnormalized}.

% =====================================================================
\section{Numerical Stability}\label{sec:numerical-stability}
% =====================================================================

Rank-aware splatting has a numerical issue that the rank-4 formulations of \cite{yang2023gs4d,duan20244drotorgaussiansplattingefficient} do not encounter. The $\varepsilon$-clamp below changes the exact conditional Gaussian; Section~\ref{subsec:stab-rank} quantifies the deviation and its limit as $\Sigmapt/\varepsilon \to \infty$.

\subsection{Degenerate temporal extent and the rank bridge}
When $n \to \pm e_0$ (equivalently $n_{1:} \to 0$), the plane approaches the static slice $\{x_0 = 0\}$ and the Gaussian has vanishing temporal extent: $\Sigmapt \to 0$. The conditioning formulas \eqref{eq:Vthree}--\eqref{eq:Sigmacond} divide by a quantity approaching zero, and the temporal weight \eqref{eq:wt} evaluates as $0/0$ at $t_0 = v_0$, which contaminates all gradients. Our random initialization (Section~\ref{sec:initialization}) places the normals away from this orientation, but the optimizer can still drive individual primitives toward it; the clamp below keeps those primitives well-behaved.

\paragraph{A shared numerical floor.}
Replace every appearance of $\Sigmapt$ and $\Sigmabt$ in a denominator by
\begin{equation}\label{eq:clamp}
  \Sigmaept = \sqrt{(\Sigmapt)^2 + \varepsilon^2}, \qquad \Sigmaebt = \sqrt{(\Sigmabt)^2 + \varepsilon^2}.
\end{equation}
We use $\varepsilon = 10^{-8}$. For $\Sigma_{tt} \gg \varepsilon$ these are indistinguishable from the unclamped quantities; for $\Sigma_{tt} \to 0$ they remain positive and smooth. The floor $\varepsilon$ is a pure numerical guard, distinct from the modeling floor $\sigma^2_{\mathrm{tmp}}$.

The following result describes how the clamp bridges the static (rank-3) and dynamic (rank-2) regimes; Figure~\ref{fig:bridge} shows the eigenvalue crossover it produces.

\begin{proposition}[Effective-rank bridge near $n = e_0$]\label{prop:bridge}
Parameterize $n$ in a neighborhood of $e_0$ by $n_{1:} = \theta u$ with $\|u\| = 1$ and $\theta \to 0$, and let $S = L_{1:} L_{1:}^T$ where $L_{1:}$ denotes the lower three rows of $L$. For generic $L$: 
\begin{itemize}
\item $\Sigmapt = O(\theta^2)$;
\item $\cworld = O(\theta)$; and 
\item along each fixed tilt direction $u$, the exact correction $\cworld\cworld^T / \Sigmapt$ remains $O(1)$, with the direction-dependent limit
\[
  \cworld\cworld^T / \Sigmapt \;\longrightarrow\; \frac{S u u^T S}{u^T S u} \qquad (\theta \to 0).
\]
\end{itemize}
In particular, the exact conditioning map does not extend continuously to $n = e_0$ without specifying a tilt direction.
\end{proposition}

\begin{proof}
With $n_0 = \sqrt{1 - \|n_{1:}\|^2}$ and $n_{1:} = \theta u$, the first row of $P_n = I - nn^T$ is
$(P_n)_{0,:} = (1 - n_0^2,\; -n_0 n_{1:}^T) = (\theta^2,\; -n_0 \theta u^T)$.
The first row of $L_n = P_n L$ is therefore dominated by the spatial coupling, $(L_n)_{0,:} = -n_0\theta\, u^T L_{1:} + O(\theta^2)$, so with $m_0^T = (L_n)_{0,:}$,
\[
  \Sigmapt = \|m_0\|^2 = \theta^2\, u^T S u + O(\theta^3),
\]
which is (1). The lower rows satisfy $(L_n)_{1:} = L_{1:} + O(\theta)$, so $\cworld = (L_n)_{1:} m_0 = -\theta\, S u + O(\theta^2)$, which is (2). Substituting,
\[
  \frac{\cworld\cworld^T}{\Sigmapt} = \frac{\theta^2\, S u u^T S + O(\theta^3)}{\theta^2\, u^T S u + O(\theta^3)} \;\longrightarrow\; \frac{S u u^T S}{u^T S u},
\]
which is (3). The limit depends on $u$, so the exact conditioning map has no continuous extension to $n = e_0$.
\end{proof}

\begin{corollary}\label{cor:bridge}
With the stabilized denominator $\Sigmaept$, the correction scales as $O(\theta^2 / \varepsilon)$ when $\theta^2 u^T S u \ll \varepsilon$, so the stabilized covariance approaches the spatial block $\Sigmathree$ and the primitive behaves like a static full-rank 3D Gaussian. When $\Sigmapt \gg \varepsilon$, the stabilized formula approaches the exact rank-2 covariance. The clamp thus bridges static full-rank behavior and dynamic rank-2 disk behavior, with crossover at $\theta \sim \sqrt{\varepsilon / u^T S u}$.
\end{corollary}

\begin{figure}[t]
\centering
\paperfig{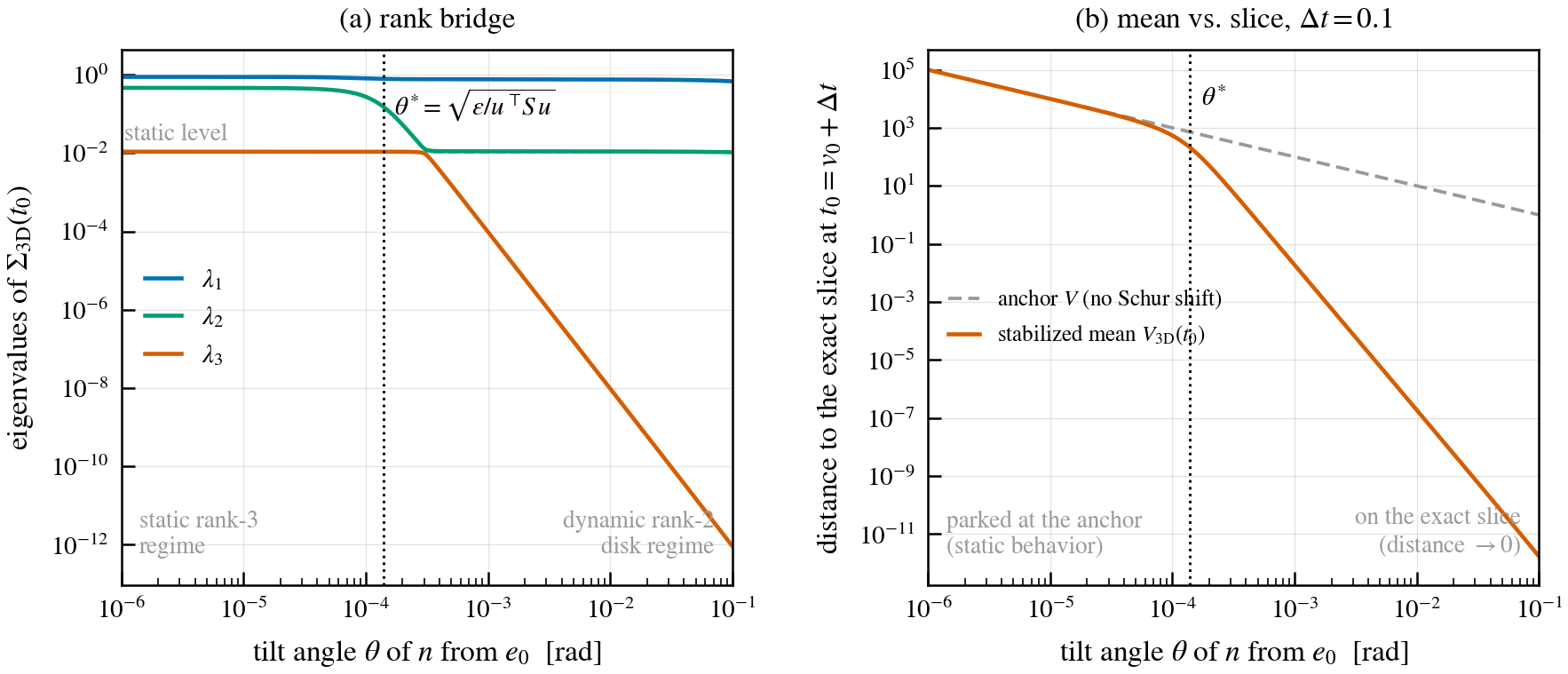}{0.78\textwidth}
\caption{Eigenvalues of the conditioned $\Sigmacond$ as the plane normal tilts away from the static orientation by angle $\theta = \angle(n, e_0)$. Below the crossover $\theta^\star = \sqrt{\varepsilon / u^T S u}$ (Corollary~\ref{cor:bridge}) all three eigenvalues are comparable and the primitive renders like a static rank-3 3DGS Gaussian; above it the smallest eigenvalue collapses and the primitive becomes a rank-2 disk.}
\label{fig:bridge}
\end{figure}

\begin{remark}[Why not use $\Sigmapt + \sigma^2_{\mathrm{tmp}}$ everywhere?]
A simpler design uses $\Sigmapt + \sigma^2_{\mathrm{tmp}}$ in both the conditioning and $w_t$, with $\sigma^2_{\mathrm{tmp}} > 0$ as a combined modeling-and-numerical floor. This couples the roles: $\sigma^2_{\mathrm{tmp}}$ must then be large enough for numerical stability, forcing every Gaussian to have a minimum visible duration. The two-floor design ($\varepsilon$ for stability, $\sigma^2_{\mathrm{tmp}}$ for modeling) keeps the roles separate.
\end{remark}

% =====================================================================
\section{Gradients and Optimization}\label{sec:gradients}
% =====================================================================

The training loop optimizes $(n_{\mathrm{raw}}, L, \mu, \alpha, \text{color})$ together with a per-Gaussian lift logit $\beta_k$ (Section~\ref{subsec:blur-trio}) by gradient descent on a video reconstruction loss.

\subsection{Autograd chain}
The map from raw parameters to rasterizer input is a composition of standard PyTorch operations:
\begin{enumerate}
  \item normalize: $n = n_{\mathrm{raw}}/\|n_{\mathrm{raw}}\|$;
  \item project: $L_n = L - n(n^T L)$;
  \item form $\Sigmafour = L_n L_n^T$;
  \item block-decompose into $\Sigmapt, \cworld, \Sigmathree$;
  \item conditioning step \eqref{eq:Vthree}--\eqref{eq:Sigmacond} with clamped denominator \eqref{eq:clamp};
  \item temporal weight \eqref{eq:wt} and effective opacity;
  \item rank lift: $\Sigmarender = \Sigmacond + \sigma^2_{\mathrm{lift},k} I_3$, with the per-Gaussian $\sigma_{\mathrm{lift},k} = \mathrm{softplus}(\beta_k)$;
  \item pass to the rasterizer.
\end{enumerate}
Steps 1--7 are pure PyTorch; step 8 is a standard 3DGS rasterizer with a documented backward. Autograd composes everything into gradients on the raw parameters. The lift is a model parameter rather than a training-only device: the learned $\sigma_{\mathrm{lift},k}$ shapes the rendered covariance at test time as well.

\subsection{Gauge directions and unconstrained parameters}
The redundant directions in $L$ (Section~\ref{subsec:effective-dofs}) require no handling: the gradient of the photometric loss vanishes along the projected-out directions, the Frobenius regularizer only decays them, and the $\OO(3)$ action leaves the full loss invariant. $\mu$, $L$, and the color are unconstrained; $\alpha$ is stored as a logit, and the per-Gaussian lift as a logit $\beta_k$ mapped through a softplus, which keeps $\sigma_{\mathrm{lift},k}$ nonnegative. None require manifold projection.

% =====================================================================
\section{Initialization and Training}\label{sec:initialization}
% =====================================================================

\subsection{Input data and preprocessing}
Input: a video sequence with known camera poses, typically from structure-from-motion.

\begin{remark}[Units]
The factor $L$ couples time and space additively, so the two axes require
a stated unit convention. Timestamps are normalized to $[0,1]$ and spatial
coordinates are kept in the capture's metric units. Every variance
hyperparameter below is calibrated in these units; rescaling one axis
requires recalibrating all of them.
\end{remark}

\subsection{Initialization}\label{subsec:initialization}
Seed points are the structure-from-motion point cloud, augmented with five times as many filler points drawn uniformly in a fixed box around the scene. One Gaussian per seed point; one plane orientation per Gaussian.

\paragraph{Plane normal.} We initialize $n$ at random: $n_{\mathrm{raw}} \sim \mathcal{N}(0, I_4)$, renormalized onto the sphere. This places every primitive away from the singular static orientation $n = \pm e_0$ from the first iteration, so temporal-tilt gradients are well-defined without a special safeguard; the soft clamp of Section~\ref{sec:numerical-stability} handles the small fraction of primitives the optimizer later drives toward $e_0$.

\paragraph{Mean.} For seed point $X_k$ we set the spatial mean $V_k = X_k$ and draw the temporal mean $v_{0,k}$ uniformly, with replacement, from the discrete training timestamps, independently of the seed point's own frame.

\paragraph{Factor.} $L_k = \rho\,\widetilde L_k$ with $\widetilde L_k$ i.i.d.\ standard normal and $\rho = \sqrt{\sigma^2_{\mathrm{init}}/3}$, $\sigma^2_{\mathrm{init}} = 0.02$, so the diagonal entries of the initial $4$D covariance have expectation $\sigma^2_{\mathrm{init}}$.

\paragraph{Opacity.} $\alpha_k$ initialized to $0.5$ (logit zero); the comparatively high value strengthens the early screen-space gradients that drive density control.

\paragraph{Per-Gaussian lift.} The lift logit is initialized to $\beta_k = \mathrm{softplus}^{-1}(10^{-2})$, i.e.\ $\sigma_{\mathrm{lift},k} = 10^{-2}$ for every primitive (Section~\ref{subsec:blur-trio}), so training departs from a model identical to one with a fixed global lift.

\subsection{Loss}\label{subsec:loss}
The training objective is
\[
  \mathcal L = 0.8\,\mathcal L_1^{\mathrm{rw}} + 0.2\,\big(1 - \mathrm{MS\text{-}SSIM}\big) + 10^{-4}\,\|L\|_F^2 + \lambda_{\mathrm{hinge}}\,\mathrm{relu}(\Sigmapt - \tau) + \lambda_{\mathrm{per}}\,\mathrm{LPIPS},
\]
with $\lambda_{\mathrm{hinge}} = 0.05$, $\tau = 0.01$, $\lambda_{\mathrm{per}} = 0.02$; the two regularizers are averaged over primitives. $\mathcal L_1^{\mathrm{rw}}$ is a residual-weighted $\ell_1$ photometric term: each pixel's $\ell_1$ error is weighted by its own (detached) channel-mean residual, normalized to unit mean over the image, so above-average-error pixels are emphasized proportionally while the overall scale of the term is preserved. The structural term is multi-scale SSIM \cite{Wang03b}. The Frobenius term lightly regularizes the raw factor, including the gauge ($n$-direction) component of $L$ that the covariance ignores (Section~\ref{subsec:effective-dofs}), and is benign. The hinge acts on the pure temporal variance (before the temporal floor and before conditioning) and discourages primitives from spanning the whole sequence. The perceptual term is LPIPS \cite{Zhang_2018_CVPR} on AlexNet features.

\subsection{Density control}
Every $K = 200$ iterations, from iteration $500$ to the end of training, we prune and split, adapted to spacetime; we do \emph{not} clone.
\begin{itemize}
  \item \textbf{Prune} if $\alpha_k < 0.005$, if the smallest eigenvalue of the pre-conditioning spatial block collapses below $10^{-6}$, or if its largest eigenvalue exceeds a fixed over-extent bound ($100$ scene units squared).
  \item \textbf{Split} Gaussians whose accumulated screen-space positional gradient exceeds $5 \cdot 10^{-5}$, decomposed into a spatial and a temporal split. The spatial split fires when the largest spatial eigenvalue exceeds $0.5$; the children are offset by one standard deviation along the major axis, and only the major-axis component of the factor is shrunk (by $1.6$), preserving the orthogonal in-plane direction and the temporal row. The temporal split fires when the floored temporal variance exceeds $0.02$ (on $[0,1]$ time); the children are offset by one standard deviation in $\mu_t$, and only the temporal row is shrunk, which leaves the conditioned spatial covariance unchanged.
  \item \textbf{Cap.} The population is capped at $8 \cdot 10^5$ primitives: at the cap, growth is suspended while pruning continues, and growth resumes if the population falls back below the cap.
\end{itemize}

\paragraph{Temporal sensitivity pruning.} On top of the local density control we apply a temporal sensitivity pruning schedule inspired by the TSP of SpeeDe3DGS \cite{TuYing2025SpeeDe3DGS}: at scheduled iterations, a fixed fraction of the population with the lowest accumulated sensitivity is removed. Our sensitivity score is simpler than theirs: SpeeDe3DGS accumulates squared gradients of the rendered image with respect to the Gaussian parameters, whereas we accumulate the squared screen-space positional gradient of the loss, $\widetilde U_k = \sum \|\nabla_{g_k} \mathcal L\|^2$, the same signal that triggers densification, summed over every step since the last pruning event. The events are at iterations $2000$, $5000$, $8000$, and $11000$ (a first event at $2000$, then every $3000$ iterations, and never within the final $1000$ iterations, so that the survivors re-optimize before evaluation), and the fraction is $30\%$. Two adaptations reflect that our primitives are time-localized rather than persistent: a primitive is eligible for pruning only once it has contributed to at least $5$ rendered frames, which protects newly split children and rarely sampled time slices; and the ranking uses the accumulated sum rather than a per-observation mean, favoring primitives that contribute across more of the timeline. Ordinary densification continues between pruning events, so the population regrows between prunes; the pruning reduces the final primitive count by up to ${\sim}2.5\times$ at unchanged quality.

\subsection{Optimizer and schedule}
Adam with parameter-group learning rates: $n_{\mathrm{raw}}$ $10^{-3}$, $L$ $5\cdot 10^{-3}$, $\alpha$ logit $5\cdot 10^{-2}$, SH color $2.5\cdot 10^{-3}$ (higher-order coefficients at $1/20$ of that), the lift logit $5\cdot 10^{-3}$, and a single group for $\mu$ on the standard 3DGS exponential position schedule ($100\times$ decay over the run). The learning rate of $n$ is held at zero for the first $3000$ iterations and then raised linearly on a window that extends past the end of training, reaching roughly half its base value by termination; the $n$ and $L$ rates are additionally decayed exponentially to $1\%$ of their base values over the run. This slow release of the plane orientation stabilizes primitives that approach the static orientation. The opacity and color rates are constant, and the active SH degree is ramped from $0$ to its maximum of $3$ over the first $3000$ iterations. Total $14{,}000$ iterations.

%=====================================================================
\section{Experiments}\label{sec:experiments}

\medskip

\subsection{Protocol}
We evaluate on the 17 HyperNeRF scenes of the Monocular Dynamic Gaussian Splatting benchmark (MonoDyGauBench) \cite{liang2025monocular}, following its protocol: the benchmark's train/test splits, images at half resolution (quarter resolution for the \texttt{tamping} scene), and LPIPS on AlexNet features. Baselines are the methods benchmarked there: the native 4D method RTGS \cite{yang2023gs4d}, the deformation-based 4DGS \cite{Wu_2024_CVPR} and DeformableGS \cite{Yang_2024_CVPR}, the spacetime-Gaussian methods STG and STG-decoder \cite{Li_2024_CVPR}, the curve-based EffGS \cite{katsumata2024compact}, and static 3DGS \cite{kerbl3Dgaussians}. Baseline numbers are taken from the benchmark, not retrained: means over the same 17 scenes and three seeds at the benchmark's full $30{,}000$-iteration budget, timed on an RTX~3090. Our numbers are single runs at $14{,}000$ iterations, timed on an NVIDIA A10G. We stop at $14{,}000$ iterations because the model has converged: per-checkpoint evaluation on three pilot scenes shows test PSNR plateauing from ${\sim}12$k iterations, with at most $+0.26$\,dB from extending to $20{,}000$ and $+0.16$\,dB from the full $30{,}000$ budget at $2.1\times$ the training cost (Fig.~\ref{fig:train-budget}, appendix). We report PSNR, SSIM, MS-SSIM, LPIPS, and training time, averaged over the scenes.

\subsection{Main results}

\begin{table}[t]
\centering
\caption{Mean results over the 17 HyperNeRF scenes of MonoDyGauBench; best per column in \textbf{bold}, our ranks among the compared methods in parentheses.}
\renewcommand{\arraystretch}{1.2}
\begin{tabular}{@{}lccccc@{}}
\toprule
Method & PSNR $\uparrow$ & SSIM $\uparrow$ & MS-SSIM $\uparrow$ & LPIPS $\downarrow$ & Train (s) $\downarrow$ \\
\midrule
4DGS (HexPlane)                  & \textbf{25.70} & \textbf{0.79} & \textbf{0.89} & \textbf{0.23} & 10171 \\
Deformable GS (MLP)  & 24.58 & 0.74  & 0.83  & 0.27  & 8855 \\
STG-decoder (TRBF)      & 23.92 & 0.73  & 0.83  & 0.34  & 7423 \\
RTGS (FourDim)                 & 22.99 & 0.71  & 0.79  & 0.35  & 11508 \\
STG (TRBF)              & 22.92 & 0.70  & 0.78  & 0.41  & 5730 \\
EffGS (Curve)     & 22.24 & 0.70  & 0.79  & 0.37  & 4120 \\
3DGS-static                  & 20.98 & 0.69  & 0.76  & 0.37  & 2587 \\
\midrule
Grassmann (ours)    & 25.05\,(2) & 0.730\,(3) & 0.863\,(2) & 0.244\,(2) & $\mathbf{1814}$\,(1) \\
\bottomrule
\end{tabular}
\label{tab:main}
\end{table}

Our method trains fastest among all compared methods, $1.4\times$ faster than static 3DGS and $4.9$ to $5.6\times$ faster than the strongest quality baselines, and renders at $11.3$\,ms per frame ($\approx 88$\,FPS) with a mean of $196$k primitives. Quality ranks second in PSNR, MS-SSIM, and LPIPS: at $25.05$\,dB, PSNR is $0.47$\,dB ahead of DeformableGS and $0.65$\,dB behind 4DGS, whose training takes $5.6\times$ longer; LPIPS is $0.014$ behind the leader. SSIM ties STG-decoder for third at the reported precision. The remaining gap to the leader is concentrated in a single class of motion, coordinated rotation of extended objects, which no straight-line primitive trajectory can represent; a quantitative per-frame analysis is deferred to future work.

% =====================================================================
\section{Discussion and Limitations}\label{sec:discussion}
% =====================================================================

We have introduced a dynamic scene representation whose primitive is a
Gaussian supported on a 3-plane in spacetime, parameterized globally
over $\Gr(3,4)$ by a unit normal and an unconstrained factor: time
conditioning returns a rank-2 surfel at every frame, the
constant-velocity motion of the disk is a property of the plane rather
than a learned deformation, and the pipeline runs on an unmodified 3DGS
rasterizer. The main limitation is built into the primitive: each plane
encodes a single constant-velocity translation, so coordinated rotation
of extended objects must be assembled from many short-lived primitives,
and this class of motion accounts for most of the remaining quality gap
to the strongest deformation baseline. Sharing motion across
primitives, rather than replacing them, is the natural extension.

 %=====================================================================
\appendix
\section{Schur Rank Lemma}\label{app:schur}
% =====================================================================

\begin{lemma}[Schur rank drop]\label{lem:schur-rank}
Let $S \in \Symp_{1+d}$ have rank $r$, blocked as
\[
  S = \begin{pmatrix} a & b^T \\ b & C \end{pmatrix}, \quad a > 0.
\]
Then $\rank(C - bb^T/a) = r - 1$.
\end{lemma}

\begin{proof}
Let $Q = \begin{pmatrix} 1 & 0 \\ -b/a & I_d \end{pmatrix}$. Then $QSQ^T = \begin{pmatrix} a & 0 \\ 0 & C - bb^T/a\end{pmatrix}$. Congruence preserves rank, so $r = 1 + \rank(C - bb^T/a)$.
\end{proof}

% =====================================================================
\section{Training-Budget Saturation}\label{app:budget}
% =====================================================================

\begin{figure}[t]
\centering
\paperfig{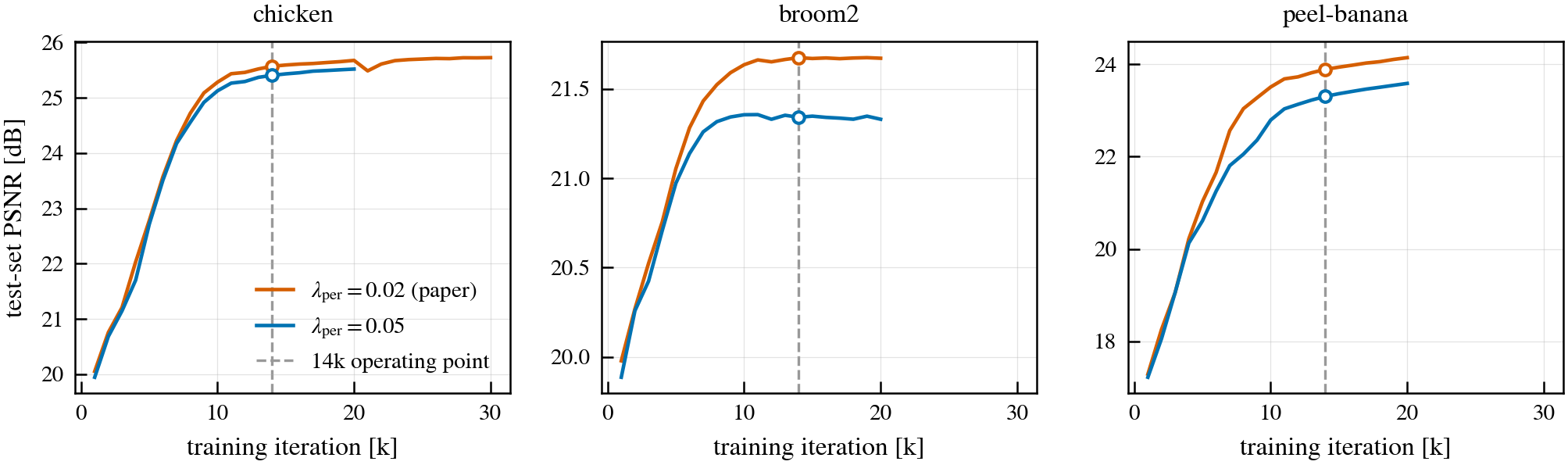}{\textwidth}
\caption{Full-test-set PSNR per $1{,}000$-iteration checkpoint on three pilot scenes (single seed). The dashed line marks the $14{,}000$-iteration operating point used throughout the paper; the step-$14{,}000$ snapshot of these runs is exactly the recipe of Section~\ref{sec:initialization}, with pruning events pinned to the same iterations. On \texttt{chicken} the $\lambda_{\mathrm{per}}=0.02$ run is resumed to the full $30{,}000$-iteration benchmark budget; the small transient at the resume point recovers within $2$k iterations.}
\label{fig:train-budget}
\end{figure}

Test PSNR plateaus from roughly iteration $12{,}000$ on all three scenes and for both perceptual-weight settings. For the paper recipe ($\lambda_{\mathrm{per}} = 0.02$), training to $20{,}000$ iterations instead of $14{,}000$ changes PSNR by $+0.11$\,dB on \texttt{chicken}, $-0.00$\,dB on \texttt{broom2}, and $+0.26$\,dB on \texttt{peel-banana}; continuing \texttt{chicken} to the benchmark's full $30{,}000$-iteration budget yields $+0.16$\,dB over the operating point at $2.1\times$ the training cost. The $\lambda_{\mathrm{per}} = 0.05$ variant behaves the same way (at most $+0.28$\,dB at $20{,}000$). LPIPS improves by at most $0.02$ over the same extensions. Table~\ref{tab:main} therefore compares against $30{,}000$-iteration baselines with a converged model, not a truncated one: the training-time advantage is a property of the operating point, not of early stopping.

\bibliographystyle{unsrt}
\bibliography{references}

\end{document}

%% file: fig_anatomy.tex
\begin{tikzpicture}[scale=1.5, >=stealth,
    font=\scriptsize,
    slice/.style={draw=oiVerm!85!black, fill=oiVerm, line width=0.5pt},
    glass/.style={fill=oiSky!30, fill opacity=0.28, draw=oiSky!60!black,
                  line width=0.4pt, draw opacity=0.45},
    leader/.style={densely dotted, oiGrey, line width=0.35pt}]
  % Okabe-Ito, matching figures/paperstyle.py
  \definecolor{oiBlue}{RGB}{0,114,178}
  \definecolor{oiSky}{RGB}{86,180,233}
  \definecolor{oiVerm}{RGB}{213,94,0}
  \definecolor{oiGreen}{RGB}{0,158,115}
  \definecolor{oiOrange}{RGB}{230,159,0}
  \definecolor{oiGrey}{RGB}{153,153,153}

  % ---- support plane mu + E_n -------------------------------------------
  \fill[oiBlue!10, draw=oiBlue!60, line width=0.5pt]
    (-1.05,0.255) -- (3.15,2.145) -- (4.65,1.365) -- (0.45,-0.525) -- cycle;
  \node[oiBlue!70!black, anchor=south east, inner sep=1pt] at (-0.35,0.62)
    {$\mu + E_n$};

  % ---- frame slices {t = t_i}: SAME window, offset = Δt · P e_t ----------
  \fill[glass] (-0.275,-0.219) -- (-0.275,2.541) -- (1.075,1.839) -- (1.075,-0.921) -- cycle;
  \fill[glass] (1.125,-0.219) -- (1.125,2.541) -- (2.475,1.839) -- (2.475,-0.921) -- cycle;
  \fill[glass] (2.525,-0.219) -- (2.525,2.541) -- (3.875,1.839) -- (3.875,-0.921) -- cycle;
  \node[oiSky!60!black, anchor=south, inner sep=2pt] at (-0.275,2.56) {slice $t_0$};
  \node[oiSky!60!black, anchor=south, inner sep=2pt] at (1.125,2.56) {$t_1$};
  \node[oiSky!60!black, anchor=south, inner sep=2pt] at (2.525,2.56) {$t_2$};

  % ---- plane ∩ slice lines (the segments below lie exactly on these) -----
  \draw[oiGrey!80, line width=0.4pt] (-0.275,0.531) -- (1.075,-0.171);
  \draw[oiGrey!80, line width=0.4pt] (1.125,1.161) -- (2.475,0.459);
  \draw[oiGrey!80, line width=0.4pt] (2.525,1.791) -- (3.875,1.089);

  % ---- leaders: segment -> its opacity on the envelope (drawn early) -----
  \draw[leader] (0.169,0.300) -- (0.40,-1.928);
  \draw[leader] (1.8,0.81)    -- (1.80,-1.35);
  \draw[leader] (3.431,1.320) -- (3.20,-1.928);

  % ---- center worldline (t, V3D(t)) — lies IN the support plane ----------
  \draw[dashed, oiGrey!90!black, line width=0.5pt, ->]
    (-0.4135,0.118) -- (4.072,1.520)
    node[anchor=north west, inner sep=2pt] {$\bigl(t,\Vthree(t)\bigr)$};

  % ---- spacetime normal n = n0 e_t + n_{1:}, all three drawn via P -------
  % dashed spatial part n_{1:} (= P e_x direction), dotted n0-connector
  % (= n0 · P e_t: horizontal, parallel to the slice offset by construction)
  \draw[dashed, oiGrey!70!black, line width=0.5pt, ->]
    (0.80,0.6995) -- (0.80,1.566);
  \node[oiGrey!70!black, anchor=west, inner sep=2pt] at (0.815,1.36) {$n_{1:}$};
  \draw[leader] (0.80,1.566) -- (0.410,1.566);
  \draw[->, line width=0.7pt] (0.80,0.6995) -- (0.410,1.566)
    node[anchor=south, inner sep=1.5pt] {$n=(n_0,\,n_{1:})$};

  % ---- ghost: t1-segment carried to t2 with NO motion ---------------------
  \draw[leader] (1.8,0.81) -- (3.2,0.81);
  \draw[dashed, oiGrey, line width=0.4pt]
    (3.2,0.81) ellipse [x radius=0.479, y radius=0.075, rotate=-27.5];
  \node[oiGrey, anchor=west, inner sep=2pt, font=\tiny] at (3.68,0.48)
    {ghost: no motion};

  % ---- displacement decomposition, a parallelogram INSIDE the t2 slice ---
  \draw[dashed, oiGrey, line width=0.4pt] (3.2,0.81) -- (3.431,0.690)
    -- (3.431,1.320);
  % normal motion, from n: along the in-slice normal n_{1:} (vertical)
  \draw[->, oiGreen, line width=0.9pt] (3.2,0.81) -- (3.2,1.44);
  \draw[oiGreen!60, line width=0.3pt] (3.99,0.85) -- (3.25,1.10);
  \node[align=left, anchor=west, inner sep=1.5pt, text=oiGreen!80!black]
    at (4.02,0.80)
    {$-\tfrac{n_0}{\|n_{1:}\|^{2}}\,n_{1:}\,\Delta t$\\[-1pt]{\color{oiGrey}(from $n$)}};
  % tangential drift, from L (cross-block c_world): along the slice line
  \draw[->, oiOrange, line width=0.9pt] (3.2,1.44) -- (3.431,1.320);
  \node[align=left, anchor=south west, inner sep=1.5pt, text=oiOrange!75!black]
    at (3.12,1.63) {$u_{\parallel}\,\Delta t$\\[-1pt]{\color{oiGrey}(from $L$)}};

  % ---- sliced Gaussian at t0, t1, t2: segments on the ∩-lines -------------
  % (thin ellipses; a rank-2 disk in the full model, a segment here because
  % one spatial dimension is suppressed)
  \begin{scope}[fill opacity=0.40]
    \draw[slice] (0.169,0.300) ellipse [x radius=0.479, y radius=0.075, rotate=-27.5];
  \end{scope}
  \begin{scope}[fill opacity=0.85]
    \draw[slice] (1.8,0.81) ellipse [x radius=0.479, y radius=0.075, rotate=-27.5];
  \end{scope}
  \begin{scope}[fill opacity=0.40]
    \draw[slice] (3.431,1.320) ellipse [x radius=0.479, y radius=0.075, rotate=-27.5];
  \end{scope}

  % ---- anchor mu ----------------------------------------------------------
  \fill (1.8,0.81) circle (1.2pt);
  \node[anchor=north west, inner sep=1pt] at (1.86,0.40) {$\mu=(v_0,V)$};

  % ---- axes triad (t horizontal: same P e_t as the slice offsets) ---------
  \begin{scope}[shift={(-1.45,-1.9)}]
    \draw[->, line width=0.5pt] (0,0) -- (0.55,0) node[anchor=west, inner sep=1pt] {$t$};
    \draw[->, line width=0.5pt] (0,0) -- (0,0.50) node[anchor=south, inner sep=1pt] {$x$};
    \draw[->, line width=0.5pt] (0,0) -- (0.40,-0.21) node[anchor=west, inner sep=0.5pt] {$y$};
  \end{scope}

  % ---- temporal envelope w_t ----------------------------------------------
  % w(t) = exp(-(t - 1.8)^2 / 2.456)  [sigma = 1.108], height 1.05
  \fill[oiBlue!12]
    (-0.35,-2.4) -- plot[domain=-0.35:4.15, samples=81]
      ({\x}, {-2.4 + 1.05*exp(-(\x-1.8)^2/2.456)}) -- (4.15,-2.4) -- cycle;
  \draw[oiBlue, line width=0.7pt]
    plot[domain=-0.35:4.15, samples=81]
      ({\x}, {-2.4 + 1.05*exp(-(\x-1.8)^2/2.456)});
  \draw[->, line width=0.5pt] (-0.35,-2.4) -- (4.6,-2.4)
    node[anchor=west, inner sep=1pt] {$t$};
  \draw[->, line width=0.5pt] (-0.35,-2.4) -- (-0.35,-1.2)
    node[anchor=south, inner sep=1pt] {$\alpha\,w_t$};
  % peak level alpha
  \draw[dashed, oiGrey, line width=0.35pt] (-0.35,-1.35) -- (1.8,-1.35);
  \node[anchor=east, inner sep=1.5pt] at (-0.35,-1.35) {$\alpha$};
  % v0 tick
  \draw[line width=0.5pt] (1.8,-2.4) -- (1.8,-2.47)
    node[anchor=north, inner sep=1.5pt] {$v_0$};
  % temporal extent: sqrt(Sigma_tt^blur) at the 1-sigma level exp(-1/2)=0.607
  \draw[<->, line width=0.5pt] (1.8,-1.763) -- (2.908,-1.763)
    node[midway, anchor=south, inner sep=1.5pt] {$\sqrt{\Sigmabt}$};
  % birth / death annotations
  \node[oiGrey, anchor=south, font=\tiny] at (0.0,-2.37) {birth};
  \node[oiGrey, anchor=south, font=\tiny] at (3.7,-2.37) {death};
\end{tikzpicture}